\documentclass[letterpaper, 10 pt, conference]{ieeeconf}  

\IEEEoverridecommandlockouts                              
\usepackage{amsmath,amssymb,amsfonts}
\usepackage[hyphens]{url}
\usepackage{algorithm}
\usepackage{algpseudocode}
\algnewcommand\algorithmicinput{\textbf{Input:}}
\algnewcommand\Input{\item[\algorithmicinput]}
\algnewcommand\algorithmicoutput{\textbf{Output:}}
\algnewcommand\Output{\item[\algorithmicoutput]}
\usepackage{graphicx}
\newtheorem{lemma}{Lemma}
\newtheorem{thm}{Theorem}
\newtheorem{definition}
{Definition}

\newtheorem{proposition}
{Proposition}
\newtheorem{remark}{Remark}
\newtheorem{corollary}{Corollary}
\newtheorem{assump}{Assumption}
\usepackage{subfigure}
\usepackage{comment}
\usepackage{cite}
\usepackage{subcaption}
\usepackage{mathtools}
\usepackage[table]{xcolor}
\usepackage{hyperref}

\title{\LARGE \bf
Maintaining Strong $r$-Robustness in Reconfigurable Multi-Robot Networks using Control Barrier Functions
}

\author{Haejoon Lee$^{1}$ and Dimitra Panagou$^{1,2}$
\thanks{*This work was supported by the Air Force Office of Scientific Research (AFOSR) under Award No. FA9550-23-1-0163.}
\thanks{$^{1}$Department of Robotics,
        University of Michigan, Ann Arbor, MI, USA
        {\tt\small haejoonl@umich.edu}}%
\thanks{$^{2}$Department of Aerospace Engineering,
        University of Michigan, Ann Arbor, MI, USA
        {\tt\small dpanagou@umich.edu }}%
\thanks{$^a$https://github.com/joonlee16/Resilient-Leader-Follower-CBF-QP}
}

\begin{document}

\maketitle

\begin{abstract}
    In leader-follower consensus, strong $r$-robustness of the communication graph provides a sufficient condition for followers to achieve consensus in the presence of misbehaving agents. Previous studies have assumed that robots can form and/or switch between predetermined network topologies with known robustness properties. However, robots with distance-based communication models may not be able to achieve these topologies while moving through spatially constrained environments, such as narrow corridors, to complete their objectives. This paper introduces a Control Barrier Function (CBF) that ensures robots maintain strong $r$-robustness of their communication graph above a certain threshold without maintaining any fixed topologies. Our CBF directly addresses robustness, allowing robots to have flexible reconfigurable network structure while navigating to achieve their objectives. The efficacy of our method is tested through various simulation and hardware experiments \href{https://github.com/joonlee16/Resilient-Leader-Follower-CBF-QP}{[code]$^a$}.
\end{abstract}
\vspace{-1mm}
\section{Introduction}
Multi-robot systems are used in tasks such as formation control, search and rescue, and object tracking~\cite{zhu2019,Kaleb_mesch,zhang2021}. One fundamental algorithm in these tasks is consensus, which enables agents to reach agreement on a common state value. Consensus can be categorized into leaderless consensus and leader-follower consensus. Leaderless consensus aims for all agents to agree on the same state value, whereas leader-follower consensus involves a subset of agents, known as followers, converging to the reference state value propagated by the set of other agents, known as leaders~\cite{ time_varying_strongly_usevitch20,ren2010, usevitch2022, dimar2008,rezaee21}. In this paper, we focus on the leader-follower consensus.

Consensus in general suffers performance degradation when misbehaving or compromised agents share incorrect, or even adversarial, information, motivating the study of resilient consensus \cite{LeBlanc13, saldana2017, usevitch2022, abbas2018, li2024, mitra2019, rezaee21,CDC2024}. The \textit{Weighted Mean-Subsequence-Reduced} (W-MSR) algorithm was introduced to allow the non-compromised (often called normal) agents reach consensus despite the presence of compromised agents \cite{LeBlanc13}. Then \cite{strongly_usevitch18} showed
 that resilient leader-follower consensus with misbehaving agents can be achieved using W-MSR under a topological condition called strong $r$-robustness \cite{mitra2019}. A similar property, namely $r$ leader-follower robustness, has been introduced in \cite{rezaee21}, which, however, requires a trustworthy leader. In this paper, we focus on strong $r$-robustness.

One common challenge in the literature of resilient leader-follower consensus is the assumption that the agents can preserve and/or switch between predetermined graph topologies with known robustness properties. However, these topologies may not be achievable as robots with distance-based communication models navigate spatially constrained environments. The problem of forming network resilience for leaderless consensus without fixed topologies has been studied in \cite{cavorsi2022,saulnier2017,guerrero2020}. However, these focus on $r$-robustness \cite{LeBlanc13}, which is not sufficient for resilient leader-follower consensus and does not directly extend to strong $r$-robustness \cite{strongly_usevitch18, rezaee21, pirani2023}. Our paper employs Control Barrier Functions to ensure that robots maintain strong $r$-robustness above some threshold.


Control Barrier Functions (CBFs) have become a popular technique for enforcing safety while considering desired objectives \cite{ames_cbf,ames_cbf_qp,garg2024, visibility_tk,dev_observer_cbf}. High-order CBFs (HOCBFs) \cite{xiao2022} handle constraints with higher relative degrees with respect to system dynamics, and learning-based CBFs have been studied in~\cite{learning_Cbf_tk,dawson_cbf,xiao2023}. While CBFs have been applied to form resilient multi-robot networks for leaderless consensus in \cite{cavorsi2022,guerrero2020}, these methods all focus on $r$-robustness and thus may not extend to resilient leader-follower consensus. In \cite{cavorsi2022}, $r$-robustness of a multi-robot network was controlled indirectly by controlling its algebraic connectivity, which could result in overly conservative formations with unnecessary edges. Authors in \cite{guerrero2020} focus on forming a specific class of $r$-robust graph. In contrast, our proposed CBF not only directly addresses strong $r$-robustness but also assumes no fixed class of graphs, offering a more flexible and general approach.

\emph{Contributions:} 
We present a CBF that maintains strong $r$-robustness of a multi-robot network to ensure resilient leader-follower consensus. We first construct HOCBFs that altogether encode strong $r$-robustness without maintaining predetermined topologies. These HOCBFs are then composed into a single valid CBF. Finally, we demonstrate the efficacy of our work with numerical and hardware experiments. 

\emph{Organization}: 
Section~\ref{sec:notation} presents the utilized notation. Section~\ref{sec:preliminaries} provides the preliminaries and problem statement. HOCBFs for robustness maintenance are constructed in Section~\ref{sec:main_hocbf}, and the composed CBF is developed in Section~\ref{sec:composition}. The experiment results are presented in Section~\ref{sec:experiment}, while conclusion is stated in Section~\ref{sec:conclusion}. 

\section{Notation}
\label{sec:notation}
Let $\mathcal G(t)=(\mathcal V, \mathcal E(t))$ be a connected, undirected graph with a vertex set $\mathcal V=\{1,\cdots, n\}$ and time-varying edge set $\mathcal E(t) \subseteq \mathcal V\times \mathcal V$. Since $\mathcal G(t)$ is undirected, an edge $(i,j) \in \mathcal E(t)$ implies $(j,i) \in \mathcal E(t)$ at the same $t$. A neighbor set of a node $i \in \mathcal V$ at time $t$ is denoted as $\mathcal N_i(t)=\{j|(i,j) \in \mathcal E(t)\}$. For simplicity, we omit $t$ when the context is clear.

The set of integers $\{1,\cdots,c\}$ is denoted as $[c]$. We denote the cardinality of a set $\mathcal S$ as $|\mathcal S|$. The set of non-negative integers, positive integers, non-negative reals, and positive reals are denoted as $\mathbb Z_{\geq 0}$, $\mathbb Z_{+}$, $\mathbb R_{\geq 0}$, and $\mathbb R_+$ respectively. Let $\|\cdot\|:\mathbb R^m\to \mathbb R_{\geq 0}$ be the Euclidean norm. A $m \times 1$ column vector of $1$ and $0$ are denoted as $\mathbf 1_m$ and $\mathbf 0_m$ respectively. Similarly, a $m \times n$ matrix of $0$ are denoted as $\mathbf 0_{m\times n}$, and an identity matrix of size $m$ is denoted as $\mathbf I_m$. An $i^{\text{th}}$ element of $\mathbf y\in \mathbb R^n$ is denoted as $y_i$. For $\mathbf y, \mathbf z \in \mathbb R^n$, $\mathbf y \geq \mathbf z$ or $\mathbf y>\mathbf z$ means $y_i \geq z_i$ or $y_i>z_i$ $\forall i\in [n]$. We define the Heaviside step function $H:\mathbb R \to \mathbb R$ and parametrized sigmoid function $\sigma_{s,q}:\mathbb R \to \mathbb R$ with $s \in \mathbb R_+$ and $q \in (0,1)$, respectively:

\begin{equation}
H(y) = \begin{cases}
    1 & \text{if } y \geq 0 \\
    0 & \text{otherwise}
\end{cases},
\label{eq:heaviside}
\end{equation}
\begin{equation}
\sigma_{s,q}(y) = \frac {1+q} {1+q^{-1}e^{-sy}} -q.
\label{eq:sigmoid}
\end{equation}
Note \eqref{eq:sigmoid} is designed such that i) $\sigma_{s,q}(y)=0$ when $y=0$, ii) $\sigma_{s,q}(y)>0$ when $y>0$, and iii) $\sigma_{s,q}(y)<0$ when $y<0$. Then, we define $H^n:\mathbb R^n \to \mathbb R^n$ and $\sigma^n_{ s,q}:\mathbb R^n \to \mathbb R^n$ 
as element-wise operations of $H$ and $\sigma_{s,q}$ respectively.
\section{Preliminaries and Problem Statement}
\label{sec:preliminaries}
\subsection{System Dynamics}

We consider a multi-robot system, in which robot $i$ has a physical state represented as $x_i(t) =\begin{bmatrix}
    p_i(t)^T & v_i(t)^T
\end{bmatrix}^T \in \mathbb R^{2m}$ with position $p_i(t)=[p_{i1}\cdots p_{im}]^T\in \mathbb R^m$ and velocity $v_i(t)=[v_{i1}\cdots v_{im}]^T \in \mathbb R^m$. An agent $i$ has double integrator dynamics of the form:
\begin{equation}
\label{eq:dynamics}
     \dot {x}_i(t)= \mathbf A_i x_i(t)+ \mathbf B_i u_i(t),
\end{equation}
where $\mathbf A_i=\begin{bmatrix}
     \mathbf 0_{m\times m} &  \mathbf I_{m} \\
      \mathbf 0_{m\times m} &  \mathbf 0_{m\times m}
     \end{bmatrix}$, $\mathbf B_i = \begin{bmatrix}
     \mathbf 0_{m\times m} \\  \mathbf I_{m} \\
     \end{bmatrix}$, $u_i(t) \in U_i \subseteq\mathbb R^m$ is the control input of a robot $i$, with $U_i$ being the input constraint set. Let $M=nm$. We denote collective states and control inputs of $n$ robots as $\mathbf x(t)=[x_1^T(t) \cdots x_n^T(t)]^T \in \mathbb R^{2M}$ and $\mathbf u(t)=[u_1^T(t) \cdots u_n^T(t)]^T \in U=U_1\times \cdots \times U_n \subseteq \mathbb R^{M}$. Also, let $\mathbf p(t)= \begin{bmatrix}
       p_1^T(t),\cdots, p_n^T(t) 
\end{bmatrix}^T\in \mathbb R^M$ and $\mathbf v(t)= \begin{bmatrix}
       v_1^T(t),\cdots, v_n^T(t) 
\end{bmatrix}^T \in \mathbb R^M$. We drop the argument $t$ when context is clear. Thus the full system dynamics is
\begin{equation}
    \dot {\mathbf x} = \mathbf A\mathbf x + \mathbf B \mathbf u, 
\label{eq:combined_dynamics}
\end{equation}
where $\mathbf A=\text{diag}([\mathbf A_1\cdots\mathbf A_n])$ and $\mathbf B=\text{diag}([\mathbf B_1\cdots\mathbf B_n])$.

 We now define the robots' communication graph $\mathcal G(t)=(\mathcal V, \mathcal E(t))$ with $\mathbf x(t)$ at each $t$, where $\mathcal V=[n]$ represents the robots and $\mathcal E(t)$ represents links between two robots. Let $\Delta_{ij}(\mathbf x(t))=||p_i(t)-p_j(t)||$ be the distance between robots $i$ and $j$ at time $t$. We consider a distance-based model
\begin{equation}
    \mathcal E(t) =\{(i,j) \mid \Delta_{ij}(\mathbf x(t)) < R\},
    \label{eq:edgeset}
\end{equation}
where $R>0$ is the communication range of robots. The adjacency matrix $A(\mathbf x(t))$ of $\mathcal G(t)$ has an entry defined as
\begin{equation}
\label{eq:adjacency}
    a_{ij}(\mathbf x(t)) = \begin{cases}
        1 & \text{if } \Delta_{ij}(\mathbf x(t)) <R \\
        0 &\text{otherwise}
    \end{cases}.
\end{equation}

\subsection{Fundamentals of Strong $r$-robustness and W-MSR}
Consider a connected network $\mathcal G(t)=(\mathcal V, \mathcal E(t))$ where $\mathcal V$ is partitioned into two static subsets: leaders $\mathcal L\subset \mathcal V$ that propagate the same reference value $f_l\in \mathbb R$ to followers $\mathcal F = \mathcal V \setminus \mathcal L$. We denote $|\mathcal L|=l$ and $|\mathcal F|=f$. At discrete times $t_N = N\tau $ for $N\in \mathbb Z_{\geq0}$ where $\tau\in \mathbb R_+$ is the update interval, agent $i$ shares its consensus value $y_i(t_N) \in \mathbb R$ with its neighbors. Robots have $y(t)=y(t_N)$ $\forall t \in [t_N, t_{N+1})$, and update with
\begin{equation}
y_i(t_{N+1})=
    \begin{cases}
        f_l, & i\in \mathcal L, \\
     \sum_{j\in \mathcal N_i\cup\{i\}} w_{ij}(t_N)y_j(t_N), & i \in \mathcal F,
    \end{cases}
    \label{eq:linear}
\end{equation}
where $w_{ij}(t_N)$ is the weight assigned to $y_j(t_N)$ by agent $i$. We assume $\exists \alpha \in(0,1)$ such that $\forall i \in \mathcal V$ $\forall N \in \mathbb Z_{\geq 0}$,
\begin{itemize}
    \item $w_{ij}(t_N)\geq \alpha$ if $j\in \mathcal N_i\cup \{i\}$, or $w_{ij}(t_N)=0$ otherwise,
    \item $\sum_{j=1}^n w_{ij}(t_N)=1$.
\end{itemize}
Since $\mathcal G$ is connected, leader-follower consensus is guaranteed through \eqref{eq:linear} \cite{ren2007}. However, this guarantee no longer holds with misbehaving agents \cite{strongly_usevitch18}. Among various models for misbehaving agents, we adopt one below for this paper:
\begin{definition}[\textbf{malicious agent}]
    \label{misbeh}
    An agent $i\in \mathcal V$ is \textbf{malicious} if it sends $y_i(t_N)$ to $\forall j\in \mathcal N_i(t_N)$ at time $t_N$ but does not update $y_i(t_{N+1})$ according to \eqref{eq:linear} at some $t_{N+1}$ \cite{LeBlanc13}.
\end{definition}

Normal agents are those who are not malicious. Various threat scopes exist to describe the number of malicious agents in a network \cite{LeBlanc13}, but for this paper we define one below:
\begin{definition}[\textbf{$\mathbf F$-local}]
    A set $\mathcal S \subset \mathcal V$ is $\mathbf F$-\textbf{local} if it
contains at most $F$ nodes as neighbors of other
nodes, i.e., $|\mathcal N_i \cap \mathcal S|\leq F$, $\forall i\in \mathcal V\setminus \mathcal S$  \cite{LeBlanc13}.
\end{definition}

In \cite{LeBlanc13}, W-MSR was introduced to guarantee leaderless consensus with up to $F$-local malicious agents by having each normal agent $i$ apply the nominal update protocol after discarding up to the $F$ highest and lowest values strictly greater and smaller than its own value $y_i(t)$ at $t$. Then,\cite{strongly_usevitch18} showed that leader-follower consensus can be achieved through W-MSR, under a topological property defined below:
\begin{definition}[$\mathbf r$-\textbf{reachable} \cite{LeBlanc13}]
    \label{reachability}
    Let $\mathcal G(t) = (\mathcal V,\mathcal{E}(t))$ be a graph at time $t$ and $\mathcal S$ be a nonempty subset of $\mathcal V$. The subset $\mathcal S$ is $\mathbf r$-\textbf{reachable} at $t$ if $\exists i\in \mathcal S$ such that $|\mathcal N_i(t) \backslash \mathcal S|\geq r$.
\end{definition}

\begin{definition}[\textbf{strongly} $\mathbf r$-\textbf{robust} \cite{mitra2019}]
    \label{robustness}
    A graph $\mathcal G(t) = (\mathcal V,\mathcal{E}(t))$ is \textbf{strongly} $\mathbf r$-\textbf{robust} with respect to $\mathcal S_1 \subset \mathcal V$ at time $t$ if $\forall \mathcal S_2 \subset \mathcal V \setminus \mathcal S_1$ such that $\mathcal S_2 \neq \emptyset$, $\mathcal S_2$ is $r$-reachable.
\end{definition}

If a network is strongly $(2F+1)$-robust $\forall t \in \mathbb R_{\geq 0}$, normal followers are guaranteed to reach consensus to $f_l$ through W-MSR with up to $F$-local malicious agents \cite{time_varying_strongly_usevitch20}.

\subsection{Problem Statement}
Let $\mathcal G(t_0)=(\mathcal V, \mathcal E(t_0))$ be a communication graph formed by a system of $n$ robots described by \eqref{eq:combined_dynamics} with state $\mathbf x(t_0)$ at the initial time $t_0 \in \mathbb R_{\geq 0}$. Let $\mathcal G(t_0)$ be at least strongly $r$-robust where $1\leq r\leq l-1$, and consider the desired control input $\mathbf u_{\text{des}}=[u_{1,\text{des}}^T\cdots u_{n,\text{des}}^T]^T$. We aim to design a control strategy so that the network remains at least strongly $r$-robust while minimally deviating from $\mathbf u_{\text{des}}(t)$ $\forall t\geq t_0$.

To solve this, we use High Order Control Barrier Functions (HOCBFs) \cite{xiao2022} and Bootstrap Percolation (BP) \cite{jason2012}. Before presenting the solution in Section~\ref{sec:composition}, we first describe these concepts in the remainder of the section.

\subsection{High-Order Control Barrier Function}
We first present High Order Control Barrier Functions (HOCBFs), which are introduced in \cite{xiao2022}. Note in this paper, we are considering time-invariant HOCBFs. Let $d^{\text{th}}$ order differentiable function $h:\mathcal D \subset \mathbb R^{2M} \to \mathbb R$ have a relative degree $d$ with respect to  system \eqref{eq:combined_dynamics}. We define $\psi_0\coloneq h(\mathbf x)$ and a series of functions $\psi_i:\mathcal D \to \mathbb R$, $i\in [d]$, as
\begin{equation}
    \psi_i(\mathbf x)\coloneq \dot \psi_{i-1}(\mathbf x)+\alpha_i(\psi_{i-1}(\mathbf x)), \quad i\in[d],
\label{eq:hocbf}
\end{equation}
where $\alpha_i(\cdot)$ are class $\mathcal K$ functions, and $\dot \psi_i(\mathbf x)=\frac {\partial \psi_i}{\partial \mathbf x}(\mathbf A\mathbf x+ \mathbf B \mathbf u )$. Each function $\psi_i$ defines a set $\mathcal C_i$ as follows:
\begin{equation}
        \mathcal C_i\coloneq\{\mathbf x\in \mathbb R^{2M} \mid \psi_{i-1}(\mathbf x)\geq 0\}, \quad i\in [d].
    \label{eq:hocbf_sets}
\end{equation}
Let $\mathcal C = \cap_{i=1}^{d}\mathcal C_i$. Now, we formally define HOCBF:
\begin{definition}
[\textbf{HOCBF} \cite{xiao2022}]
    Let $\psi_i(\mathbf x)$, $i\in\{0,\cdots,d\}$, be defined by \eqref{eq:hocbf} and $\mathcal C_i$, $i\in[d]$, be defined by \eqref{eq:hocbf_sets}. Then, a function $h:\mathcal D \to \mathbb R$ is a \textbf {HOCBF} of relative degree $d$ for system \eqref{eq:combined_dynamics} if there exist differentiable class $\mathcal K$ functions $\alpha_1,\cdots,\alpha_{d}$ such that
    \begin{equation}
        \sup\limits_{\mathbf u \in U}  \big[\psi_{d}(\mathbf x)\big]\geq 0,\quad\forall \mathbf x \in \mathcal C.
    \end{equation}
\end{definition}
\begin{thm}
    Given the HOCBF $h$ and its safety set $\mathcal C$, if $x(t_0)\in \mathcal C$, any Lipschitz continuous controller $\mathbf u(\mathbf x) \in K_{\text{hocbf}}(\mathbf x)=\{\mathbf u\in U\mid \psi_{d}(\mathbf x)\geq0\}$ renders $\mathcal C$ forward invariant for system \eqref{eq:combined_dynamics} \cite{xiao2022}.
    \label{thm:hocbf}
\end{thm}

HOCBF is a generalization of Control Barrier Function (CBF) with $d=1$ \cite{ames_cbf, ames_cbf_qp, xiao2022}. In this paper, we present a CBF that enforces multi-robot system to form strongly $r$-robust network, where $r\in\mathbb Z_+$ is a user-given parameter.

\subsection{Bootstrap Percolation}
Now, we introduce bootstrap percolation (BP) \cite{jason2012}. Let $\mathcal G(t)=(\mathcal V, \mathcal E(t))$ be a communication graph formed by $n$ robots with physical states $\mathbf x(t)$ at time $t$. Given user-defined threshold $r\in \mathbb Z_+$ and initial set of active nodes $\mathcal L=[l] \subset \mathcal V$, BP models the spread of \textit{activation} of nodes of $\mathcal G$ from $\mathcal L$. In BP, each node is either \textit{active} (activation state $1$) or \textit{inactive} (activation state $0$). The process iteratively activates nodes with at least $r$ active neighbors until no further activations occur, terminating in at most $|\mathcal F|=f$ iterations. Once activated, a node stays active until termination. If all nodes in $\mathcal V$ become active through BP with threshold $r$ and initial active set $\mathcal L$, then we say $\mathcal L$ percolates $\mathcal G$ with threshold $r$.

We now mathematically define the process of BP up to a given number of iterations $\delta$ where $1\leq\delta\leq f$. Let $\pi^r_j(k)$ denote the activation state of node $j$ at iteration $k \in \{0,\cdots, \delta\}$ of BP, which is updated using the rule below
\begin{equation}
\pi^r_j(k+1)=
\begin{cases}
   1 & \text{if $\pi^r_j(k)=1$}\\
   1 & \text{if $\sum_{i \in \mathcal N_j} \pi^r_i(k) \geq r$}\\
   0 & \text{if $\sum_{i \in \mathcal N_j} \pi^r_i(k) < r$}
\end{cases}.
\label{eq:percolation}
\end{equation}
We now define an activation state vector of all nodes at iteration $k$ as $\pi^r(k)=[\pi^r_1(k)\cdots \pi^r_{n}(k)]^T$. This can be decomposed into
    $\pi^r(k) = \begin{bmatrix}
    \pi^r_\mathcal L(k)^T & \pi^r_\mathcal F(k)^T
\end{bmatrix}^T$,
where $\pi^r_\mathcal L(k) = \begin{bmatrix}
    \pi^r_{1}(k) \cdots \pi^r_l(k)
\end{bmatrix}^T$ and $\pi^r_\mathcal F(k) = \begin{bmatrix}
    \pi^r_{l+1}(k) \cdots \pi^r_n(k)
\end{bmatrix}^T$ are activation state vectors of leaders and followers at iteration $k$, respectively. Then, $\pi^r_\mathcal L(0)=\mathbf 1_l$ and $\pi^r_\mathcal F(0)=\mathbf 0_f$. Note that $\pi^r(\delta)= \mathbf 1_{n}$ implies $\mathcal L$ percolates $\mathcal G$ with threshold $r$. Now, we present a lemma from \cite{mitra2019}:
\begin{lemma}
Given a graph $\mathcal G$ and threshold $r\in \mathbb Z_+$, an initial set $\mathcal L$ percolates $\mathcal G$ with threshold $r$ if and only if $\mathcal G$ is strongly $r$-robust with respect to $\mathcal L$ \cite{mitra2019}.
\label{lem:first}
\end{lemma}
Using Lemma~\ref{lem:first}, we derive the following corollary:
\begin{corollary}
    Given a graph $\mathcal G$, threshold $r\in \mathbb Z_+$, and initial set $\mathcal L$, the process of BP activates every follower $i \in \mathcal F$ if and only if $\mathcal G$ is strongly $r$-robust with respect to $\mathcal L$. 
    \label{col:follower_only}
\end{corollary}

With Corollary~\ref{col:follower_only}, $\pi^r_\mathcal F(\delta)=\mathbf 1_f$ directly implies that $\mathcal G(t)$ is strongly $r$-robust at time $t$. We use this intuition to build HOCBFs in the next section.



\section{HOCBFs for Strong $r$-Robustness}
\label{sec:main_hocbf}

Here, we construct HOCBFs that create sufficient connections for a communication graph $\mathcal G(t)=(\mathcal V, \mathcal E(t))$ to be strongly $r$-robust with respect to its leader set $\mathcal L=[l]\subset \mathcal V$. Let $\pi^r(k)$ be an activation state vector of all nodes at iteration $k\in\{0,\cdots, \delta\}$ of BP where $1\leq \delta\leq f$. Let $A(\mathbf x)$ be an adjacency matrix of $\mathcal G$ at time $t$. Using the reasoning of BP, we recursively represent $\pi^r_\mathcal F(k)$ for $k\in[\delta]$ with $\pi^r_\mathcal F(0)=\mathbf 0_{f}$:
\begin{equation}
   \pi^r_\mathcal F(k)
   = H^f\Big(\underbrace{\begin{bmatrix}
       \mathbf 0_{f\times l} & \mathbf I_f
   \end{bmatrix}A(\mathbf x) \begin{bmatrix}
       \mathbf 1_{l} \\
       \pi^r_\mathcal F(k-1) 
   \end{bmatrix}}_K- r\mathbf 1_f\Big)
    \label{eq:original}
\end{equation}
where $i^{\text{th}}$ element of $K\in \mathbb R^f$ indicates how many active neighbors a follower $i$ has at iteration $k$.

\subsection{Continuous Representation of the Activation States}

However, $\pi^r_\mathcal F(\delta)$ as a result of \eqref{eq:original} for $k\in[\delta]$ is not continuous and thus not suitable for HOCBF. Therefore, we construct a smooth approximation of $\pi^r_\mathcal F(\delta)$. We first define a parameter set $\mathcal B=\{s, s_A, q, q_A\}$ where
$s,s_A\in \mathbb R_+$ and $q,q_A\in (0,1)$. We also define $\bar A(\mathbf x)$ whose entry is
\begin{equation}
    \bar a_{ij}(\mathbf x) = \begin{cases}\sigma_{s_{A},q_{A}}\big(\big(R^2-\Delta_{ij}(\mathbf x)^2\big)^3\big) & \text{if }  \Delta_{ij}(\mathbf x) < R\\
    0 &\text{otherwise}
\end{cases}.
\label{eq:smooth_adjacency}
\end{equation}
 Note that $\bar a_{ij}(\mathbf x)\to 1$ as robots $i$ and $j$ get closer in distance, while $\bar a_{ij}(\mathbf x)\to 0$ as they get further apart. Also with $\sigma_{s_A,q_A}(0)=0$ and $(R^2-\Delta_{ij}(\mathbf x)^2)^3$, one can verify that it is twice continuously differentiable $\forall \mathbf x\in \mathbb R^{2M}$. By replacing $A(\mathbf x)$ and $H^f$ in \eqref{eq:original} with $\bar {A}(\mathbf x)$ and $\sigma_{s, q}^f$, respectively, we construct $\bar \pi^r_\mathcal F(k) = \begin{bmatrix}
    \bar \pi^r_{l+1}(k) \cdots \bar \pi^r_n(k)
\end{bmatrix}^T$, $k\in[\delta]$, where
\begin{equation}
   \bar \pi^r_\mathcal F(k)
  = \sigma^f_{ s,q}\Big(\begin{bmatrix}
       \mathbf 0_{f\times l} & \mathbf I_f
   \end{bmatrix}\bar A(\mathbf x) \begin{bmatrix}
       \mathbf 1_{l} \\
       \bar \pi^r_\mathcal F(k-1)
   \end{bmatrix}- r\mathbf 1_f\Big)
    \label{eq:sigmoid_percolation}
\end{equation}
with $\bar \pi^r_\mathcal F(0)=\mathbf 0_{f}$. We now show that $\bar \pi^r_{\mathcal{F}}(\delta)$ also gives sufficient information to determine if $\mathcal{G}$ is strongly $r$-robust.

\begin{lemma} Let $n'\in \mathbb Z_+$, $s'\in \mathbb R_+$, and $q'\in (0,1)$. Then,
 $H^{n'}(\mathbf y)> \sigma^{n'}_{ s',q'}(\mathbf y)$ $\forall \mathbf y \in \mathbb R^{n'}$.
 \label{lem:under}
\end{lemma}
\begin{proof}
   Since $H^{n'}(\cdot)$ and $\sigma^{n'}_{s',q'}(\cdot)$ are element-wise operations of $H(\cdot)$ and $\sigma_{s',q'}(\cdot)$, we just need to prove $H(y_i)> \sigma_{s',q'}(y_i)$. Dropping the subscript $i$ and superscript $'$, we get $\sigma_{s,q}(y) = \frac {1+q} {1+e^{-sy}(1/q)}-q= \frac {q(1+q)} {q+e^{-sy}}-q= \frac {q(1-e^{-sy})} {q+e^{-sy}}$. We know i) $s\in \mathbb R_+$ and ii) $q\in(0,1)$. Thus, 
   
   If $y \geq 0$, $H(y)=1>\sigma_{s,q}(y)\geq 0$ as $0\leq q(1-e^{-sy})<q$. 
   
   If $y < 0$, $H(y)=0>\sigma_{s,q}(y)$ as $e^{-sy}>1$.
\end{proof}
Using Lemma~\ref{lem:under}, we  characterize $\bar \pi^r_{\mathcal{F}}(\delta)$ in terms of $\mathcal G$'s robustness in the following proposition:
\begin{proposition}
Let $\mathcal G(t)=(\mathcal V, \mathcal E(t))$ be a communication graph of a system of $n$ robots with states $\mathbf x(t)$. Let $\mathcal L=[l]\subset \mathcal V$ and $\mathcal F=\mathcal V\setminus\mathcal L$ be leader and follower sets. Let $\delta \in \mathbb Z_+$ such that $\delta\leq f$. Also, let $\bar \pi^r_\mathcal F(\delta)$ be computed from \eqref{eq:sigmoid_percolation} for $k\in [\delta]$ using $\mathcal B=\{s, s_A, q, q_A\}$ with
$s,s_A\in \mathbb R_+$ and $q,q_A\in (0,1)$. Then, $\mathcal G(t)$ is strongly $r$-robust with respect to $\mathcal L$ at time $t$ if $\bar \pi^r_\mathcal F(\delta) \geq \mathbf 0_{f}$. \label{prop:approx}
\end{proposition}
\begin{proof} Corollary~\ref{col:follower_only} shows that $\mathcal G$ is strongly $r$-robust with respect to $\mathcal L$ if $\pi^r_\mathcal F(\delta)=\mathbf 1_{f}$. Now, we compare $\pi^r_\mathcal F(\delta)$ with $\bar \pi^r_\mathcal F(\delta)$. With Lemma~\ref{lem:under}, $\bar a_{ij}(\cdot)$ defined at \eqref{eq:smooth_adjacency} always under-approximate $a_{ij}(\cdot)$ defined at \eqref{eq:adjacency}. Again, with Lemma~\ref{lem:under},  $\sigma^n_{s,q}(\cdot)< H^n(\cdot)$ for the same argument. Therefore, with $\bar \pi^r_\mathcal F(0) = \pi^r_\mathcal F(0)$, $\bar \pi^r_\mathcal F(k) <\pi^r_\mathcal F(k)$ $\forall k \in [\delta-1]$. Since i) $\bar \pi^r_\mathcal F(\delta-1) < \pi^r_\mathcal F(\delta-1)$ and ii) $H^n(\mathbf 0_{n})=\mathbf 1_{n}$ while $\sigma^n_{ s,q}(\mathbf 0_{n})=\mathbf 0_{n}$, $\bar \pi^r_\mathcal F(\delta) \geq \mathbf 0_{f}\to \pi^r_\mathcal F(\delta)=\mathbf 1_{f}$.
\end{proof}
\begin{remark}
Proposition~\ref{prop:approx} ensures that $\mathcal G(t)$ is strongly $r$-robust with respect to $\mathcal L$ if $\bar \pi^r_\mathcal F(\delta)\geq \mathbf 0_{f}$, but it also introduces a tradeoff. Since $\bar a_{ij}(\cdot)<1$, with strongly $r_0$-robust $\mathcal G(t_0)$, $\bar \pi^r_\mathcal F(1)<\mathbf 0_{f}$ if $l=r_0$. This means $\bar \pi^r_\mathcal F(\delta)<\mathbf 0_{f}$ unless $l>r_0$. Thus, we only consider $r_0\leq l-1$ in this paper.
 
 \label{remark:robustness}
\end{remark}
\subsection{HOCBF for Robustness Maintenance}
Let $\epsilon\in \mathbb R_+$ be a small constant. Using Proposition~\ref{prop:approx}, we construct the following candidate HOCBFs:
\begin{equation}
    h_r(\mathbf x)=
    \begin{bmatrix}
    h_{r,1}(\mathbf x) \\ \vdots \\ h_{r,f}(\mathbf x)
    \end{bmatrix} = \begin{bmatrix}
    \bar \pi^r_{l+1}(\delta) -\epsilon\\ \vdots \\ \bar \pi^r_n(\delta)-\epsilon
    \end{bmatrix} = \bar \pi^r_\mathcal F(\delta)-\epsilon\mathbf 1_f.
    \label{eq:all_h}
\end{equation}
 Note $h_{r,c}(\mathbf x)$, $c\in[\delta]$, is twice continuously differentiable and has a relative degree $2$ to dynamics \eqref{eq:combined_dynamics}. Thus, we have: 
\begin{equation}
    \begin{split}
    \psi_{c,0}(\mathbf x)\coloneq &h_{r,c}(\mathbf x), \\
    \psi_{c,1}(\mathbf x)\coloneq &\dot \psi_{c,0}(\mathbf x)+\eta_{c,1}\psi_{c,0}(\mathbf x),\\
    \psi_{c,2}(\mathbf x)\coloneq &\dot \psi_{c,1}(\mathbf x)+\eta_{c,2}\psi_{c,1}(\mathbf x),
\end{split}
\label{eq:split_hocbf}
\end{equation}
where $\eta_{c,1},\eta_{c,2}\in \mathbb R_+$. Let $\mathcal C_{c,i}=\{\mathbf x\in \mathbb R^{2M}\mid \psi_{c,i-1}(\mathbf x) \geq 0\}$, $i\in[2]$. When $\mathbf x \in \mathcal C_c =\cap_{i=1}^2\mathcal C_{c,i}$, the agents are arranged so that follower $c$ can maintain sufficient connections to preserve the network's strong 
$r$-robustness at time $t$. Since strong $r$-robustness requires all agents to be sufficiently connected, we need $\mathbf x \in \mathcal C = \bigcap_{c=1}^{f} \mathcal C_c$. Also, we aim to maintain the initial level of robustness or one lower, i.e., $\mathbf x(t_0)\in \mathcal C$. We now assume:
\begin{assump}    \label{assump:collision}
    Inter-agent collision avoidance is enforced.
\end{assump}
\begin{assump}    \label{assump:bounded_input}
    The control input is unbounded: $U =\mathbb R^{M}$.
\end{assump}
Also, we define an extreme agent, adopting from \cite{capelli2020}. Let agent $I \in \mathcal V$ with position $p_I=[p_{I1},\cdots, p_{Im}]^T\in \mathbb R^m$ be an extreme agent at dimension $b$ if $\exists b\in[m]$ such that $p_{Ib} > p_{jb}$ $\forall j\in \mathcal V\setminus\{I\}$ or $p_{Ib} < p_{jb}$ $\forall j\in \mathcal V\setminus\{I\}$. In words, extreme agents at dimension $b$ have the strictly largest or smallest $b^{\text{th}}$ position component.
We assume extreme agents always exist. In fact, under Assumption~\ref{assump:collision}
, their absence would require at least $2m$ agent pairs to have the same position components, which can be easily avoided. We do not address the case formally here. A potential solution is adding small perturbations to a robot's position or defining a different control law for these configurations; fully addressing this limitation remains future work. Now we provide a lemma:

\begin{lemma}
    Let all conditions in Proposition~\ref{prop:approx} hold.
    Let $h_r(\mathbf x)$ \eqref{eq:all_h} be computed with $\mathbf x \in \mathcal C$. Then, $\exists I \in [n], b \in [m]$ such that $\frac {\partial h_{r,c}}{\partial p_{Ib}}(\mathbf x)>0$ $\forall c\in[f]$ or $\frac {\partial h_{r,c}}{\partial p_{Ib}}(\mathbf x)<0$ $\forall c\in[f]$.
    \label{lem:same_sign}
\end{lemma}
\begin{proof}Let $D=\delta-1$, $c\in[f]$, and $j=c+l$. With $\bar \pi^r_j(0)=0$ $\forall j\in \mathcal F$, we write $\bar \pi^r_j(k+1)$, $k\in\{0,\cdots,D\}$, as
\begin{equation}
    \bar \pi^r_j(k+1) = \sigma_{s,q}\underbrace{\Big(\sum_{i=1}^l\bar a_{ji}+ \sum_{i=l+1}^n\bar a_{ji}\bar \pi^r_i(k)-r\Big)}_{B_{j,k}}.
    \label{eq:rewrite}
\end{equation}
Considering extreme agent $I\in \mathcal V$ at dimension $b\in[m]$, the partial derivative of $h_{r,c}(\mathbf x)$ with respect to $p_{Ib}$ is
\begin{multline}
    \frac {\partial h_{r,c}}{\partial p_{Ib}} = \frac {\partial \sigma_{s,q}}{\partial B_{j,D}}\Big(\sum_{i=1}^l\frac {\partial \bar a_{ji}}{\partial p_{Ib}}+\\\sum_{i=l+1}^n \Big(\frac {\partial \bar a_{ji}}{\partial p_{Ib}}\bar \pi^r_i(D)+ \bar a_{ji}\frac {\partial \bar \pi^r_i(D)}{\partial p_{Ib}}\Big)\Big),
    \label{eq:derivative1}
\end{multline}
where, with $n_{ij}(\mathbf x)=[\Delta_{ij}(\mathbf x)]^2$,
\begin{multline}
        \frac {\partial \bar \pi^r_i(k+1)}{\partial p_{Ib}} = \frac {\partial \sigma_{s,q}}{\partial B_{i,k}}\Big(\sum_{g=1}^l\frac {\partial \bar a_{ig}}{\partial p_{Ib}}+\\\sum_{g=l+1}^n \Big(\frac {\partial \bar a_{ig}}{\partial p_{Ib}}\bar \pi^r_g(k)+ \bar a_{ig}\frac {\partial \bar \pi^r_g(k)}{\partial p_{Ib}}\Big)\Big),
    \label{eq:derivative2}
\end{multline} 
\begin{equation}
    \frac {\partial \bar a_{ji}}{\partial p_{Ib}} = \begin{cases}
    \frac{\partial \bar a_{ji}}{\partial n_{ji}}\frac{\partial n_{ji}}{\partial p_{Ib}} & \text{if } I\in\{i,j\} \\
    0 & \text{otherwise}
    \end{cases}.
    \label{eq:sigma_der}
\end{equation}
Now we list some facts:
\begin{itemize}
    \item  $\frac {\partial \sigma_{s,q}}{\partial B_{j,D}}(B_{j,D})>0$, since $\sigma_{s,q}$ is strictly monotonically increasing and continuously differentiable functions.
    \item If $I\in\{i,j\}$ and $(i,j)\in \mathcal E$, $\frac{\partial \bar a_{ji}}{\partial n_{ji}}(\mathbf x)>0$.
    \item If $I\in\{i,j\}$ and $i\neq j$, $\frac{\partial n_{ji}}{\partial p_{Ib}}(\mathbf x)=(\pm2(p_{jb}-p_{ib}))>0$ $\forall i,j\in[n]$ or $\frac{\partial n_{ji}}{\partial p_{Ib}}(\mathbf x)<0$ $\forall i,j\in[n]$.
    \item If $D>1$, $\bar \pi^r_i(D)> \epsilon$ $\forall i\in\mathcal F$ as $\mathbf x \in \mathcal C$. 
\end{itemize}
Since $|\mathcal N_I|\geq 1$, we can infer that if $I\in\{i,j\}$ and $i\neq j$, $\exists i,j \in \mathcal V$ that \eqref{eq:sigma_der} is nonzero, and the nonzero values are either all positive or negative. If $D=0$, $\forall j\in \mathcal F$ $\exists i\in [l]$ that $\frac{\partial \bar a_{ji}}{\partial n_{ji}}(\mathbf x)>0$, so $\frac {\partial h_{r,c}}{\partial p_{Ib}}>0$ or $\frac {\partial h_{r,c}}{\partial p_{Ib}}<0$ $\forall c\in[f]$. If $D>0$, either $\frac {\partial \bar \pi^r_i(k)}{\partial p_{Ib}}>0$ $\forall i\in[f]$ $\forall k\in[D]$ or $\frac {\partial \bar \pi^r_i(k)}{\partial p_{Ib}}<0$ $\forall i\in[f]$ $\forall k\in[D]$. Then, $\forall c\in[f]$ i) \eqref{eq:derivative1} has nonzero terms, and ii) all of them for all $c$ have the same sign. 
\end{proof}

Now we use Lemma~\ref{lem:same_sign} to validate our HOCBFs:
\begin{thm}
Let all conditions in Lemma~\ref{lem:same_sign} hold.
    Also, let $\psi_{c,2}(\mathbf x)$ be from \eqref{eq:split_hocbf}. Then, $h_{r,c}(\mathbf x)$ is a HOCBF of relative degree $2$ for system \eqref{eq:combined_dynamics} $\forall c\in[f]$.
    \label{thm:one_valid}
\end{thm}
\begin{proof}
    For each $c\in[f]$, we need to show that 
   \begin{equation}
    \begin{split}
        \sup\limits_{\mathbf u \in U}  \big[\psi_{c,2}(\mathbf x)\big]\geq 0,\quad\forall \mathbf x \in \mathcal C,\\
    \end{split}
    \label{eq:sup_h2}
    \end{equation}
where 
\begin{equation}
    \psi_{c,2}(\mathbf x) =\underbrace{\mathbf v^T \frac {\partial^2 h_{r,c}}{\partial \mathbf p^2}\mathbf v + \frac {\partial h_{r,c}}{\partial \mathbf p}\mathbf u + \eta_{c,1}\frac {\partial h_{r,c}}{\partial \mathbf p}\mathbf v}_{\dot \psi_{c,1}(\mathbf x)} + 
    \eta_{c,2}\psi_{c,1}(\mathbf x).
    \label{eq:second_der}
\end{equation}
With Assumption~\ref{assump:bounded_input}, it suffices to show that $\frac {\partial h_{r,c}}{\partial \mathbf p}(\mathbf x)\neq \mathbf 0_{M}^T$ $\forall x\in \mathcal C$ $\forall c\in[f]$, which is given in Lemma~\ref{lem:same_sign}.
\end{proof}

Note the HOCBFs \eqref{eq:all_h} are functions of robots' states $\mathbf x$ and desirable level of strong $r$-robustness. They directly consider robustness without imposing maintenance of fixed topologies, allowing robots to have more flexible formations.

\section{CBF Composition}
\label{sec:composition}
\subsection{Composition of Multiple HOCBFs}
Given $f$ HOCBFs \eqref{eq:all_h}, we now seek to compose them into one CBF. Composition methods are studied in \cite{black2023, egerstedt2018, glotfelter2017, breeden2023}, and in this paper we use a form inspired by \cite{black2023}. We define a candidate CBF $\phi_r:\mathbb R^{2M}\times \mathbb R_+^f\to \mathbb R$ as follows:
\begin{equation}
    \phi_r(\mathbf x, \mathbf w)=1-\sum_{c=1}^{f}e^{-w_c\psi_{c,1}(\mathbf x)},
    \label{eq:consolidated}
\end{equation}
where $\mathbf w=\{w_c \in \mathbb R_+\mid c\in[f]\} \in \mathbb R_+^f$. Its structure implies that $\phi_r(\mathbf x,\mathbf w)\geq 0$ only when $\psi_{c,1}(\mathbf x)\geq0$ $\forall c\in[f]$. Then, the safety set $\mathcal S(\mathbf w)=\{\mathbf x \in \mathbb R^{2M} \mid \phi_r(\mathbf x,\mathbf w)\geq0\} \subset \mathcal C$. How close $\mathcal S(\mathbf w)$ gets to $\mathcal C$ depends on the values of $\mathbf w$. Note in \cite{black2023}, the values of $\mathbf w$ are adjusted to validate their composed CBF. However, due to the nature of our HOCBFs \eqref{eq:all_h}, we do not need to adjust them, as shown below:
\begin{thm}
    Let all conditions in Theorem~\ref{thm:one_valid} hold. Let 
\begin{equation}
        \psi(\mathbf x)=\begin{bmatrix}
            \psi_{1,1}(\mathbf x) & \cdots & \psi_{f,1}(\mathbf x)
        \end{bmatrix}
    \end{equation} 
    and $\mathbf w\in \mathbb R^f_+$. Then, \eqref{eq:consolidated} is a valid CBF with system \eqref{eq:combined_dynamics}.
    \label{thm:combined_valid}
\end{thm}
\begin{proof}
We know \eqref{eq:consolidated} is continuously differentiable and has a relative degree $1$ with system \eqref{eq:combined_dynamics}.
Now, we prove:
\begin{equation}
    \begin{split}
        \sup\limits_{\mathbf u \in U}  \Big[\frac {\partial \phi_r}{\partial \mathbf x} (\mathbf A \mathbf x + \mathbf B \mathbf u)\Big]\geq 0,\quad\forall \mathbf x \in \mathcal C.
    \end{split}
    \label{eq:hdot}
    \end{equation}
We can rewrite $\frac {\partial \phi_r}{\partial \mathbf x} (\mathbf A \mathbf x + \mathbf B \mathbf u)$ in \eqref{eq:hdot} as $\frac {\partial \phi_r}{\partial \psi}\dot {\psi}$ where 
\begin{equation}
    \frac {\partial \phi_r}{\partial \psi}(\mathbf x) = \begin{bmatrix}
        w_1e^{-w_1\psi_{1,1}(\mathbf x)} &\cdots &w_fe^{-w_f\psi_{f,1}(\mathbf x)}
    \end{bmatrix},
    \label{eq:weight_der}
\end{equation}
\begin{equation}
    \dot \psi(\mathbf x) = \begin{bmatrix}
       \dot \psi_{1,1}(\mathbf x)^T & \cdots & \dot \psi_{f,1}(\mathbf x)^T
    \end{bmatrix}^T.
    \label{eq:psi_dot}
\end{equation}
 
From \eqref{eq:second_der}, the coefficient of the control input $\mathbf u$ in $\dot \psi$ is $\frac {\partial h_r}{\partial \mathbf p}$. Then, with Assumption~\ref{assump:bounded_input}, we just need to show that
   \begin{equation}
    \begin{split}
        \frac {\partial \phi_r}{\partial \psi}(\mathbf x)\frac {\partial h_r}{\partial \mathbf p}(\mathbf x) \neq \mathbf 0_{M}^T
    \end{split}, \quad  \forall \mathbf x \in \mathcal S(\mathbf w) \subset \mathcal C.
    \label{eq:ucoeff}
    \end{equation}
Since i) $\exists$ column in $\frac {\partial h_{r}}{\partial \mathbf p}$ where all nonzero entries have the same sign (Lemma~\ref{lem:same_sign}), and ii) $\frac {\partial \phi_r}{\partial \psi}(\mathbf x)>\mathbf 0_{f}^T$, \eqref{eq:ucoeff} holds.
\end{proof}

Using our CBF \eqref{eq:consolidated}, the robots can maintain robustness without specific network topologies unlike the assumptions in \cite{time_varying_strongly_usevitch20, strongly_usevitch18,rezaee21,li2024}, while minimally deviating from $\mathbf u_{\text{des}}$. Note it only takes robots' states and desirable robustness level. This provides robots with flexible reconfigurable network formations such that they deviate from $\mathbf u_{\text{des}}$ only when their motions conflict with desire to maintain robustness.

\section{Experimental Results}
\label{sec:experiment}
Now we demonstrate our work through simulations and hardware experiments. For practical deployments, we incorporate inter-agent and obstacle collision avoidance into our controller. Modeling robots and obstacles as point masses, robots must maintain distances of at least $\Delta_d$ from each other and $\Delta_o$ from obstacles. These constraints are encoded as additional CBFs and composed with \eqref{eq:consolidated} into a candidate CBF $Y:\mathbb R^{2M} \times \mathbb R_+^{f} \to \mathbb R$, which defines a safety set $\mathcal W=\{\mathbf x \in \mathbb R^{2M}\mid Y(\mathbf x, \mathbf w)\geq 0\}\subset \mathcal S(\mathbf w)$. Note that establishing $Y(\mathbf x,\mathbf w)$ as a valid CBF is an open problem and is not addressed in this paper. We now construct our CBF-based QP (CBF-QP) controller:
\begin{equation}
\begin{split}
    \mathbf u(\mathbf x)= &\arg \min\limits_{\mathbf u\in \mathbb R^{M}}||\mathbf u_{\text{des}}-\mathbf u||^2 
    \\
     \text{s.t. } 
     & \frac {\partial Y}{\partial \mathbf x}\big(\mathbf A \mathbf x + \mathbf B \mathbf u\big)\geq -\alpha_T(Y(\mathbf x,\mathbf w)),
\end{split}
\label{eq:total_qp}
\end{equation}
where $\alpha_T(\cdot)$ is a class $\mathcal K$ function. 

In this setup, robots with states $\mathbf x$ form a network using \eqref{eq:edgeset} and constantly emit LED colors corresponding to their scalar consensus values in $[0,1]$. Robots share their values with neighbors every $\tau=0.5$ seconds. Normal leaders all maintain the same value $f_l\in[0,1]$ and thus display the same LED color, while normal followers start with random values in $[0,1]$ and adjust them through W-MSR to match leaders' LED colors (consensus values). Malicious agents randomly choose and share values in $[0,1]$ every $\tau$ seconds, updating their LEDs accordingly. We evaluate efficacy of our CBF by observing whether followers successfully match their LED colors to the normal leaders' LED colors despite malicious agents. We set $\delta=4$, $\epsilon=10^{-4}$, $\Delta_d = 0.3$ m, and $\Delta_o \in\{0.4,0.7\}$ m, based on obstacle sizes. Codes and video are available \href{https://github.com/joonlee16/Resilient-Leader-Follower-CBF-QP}{here}\footnote{\label{note1}https://github.com/joonlee16/Resilient-Leader-Follower-CBF-QP}.

\subsection{Simulation}
\begin{figure}[H]
    \centering
\includegraphics[width=.49\textwidth]{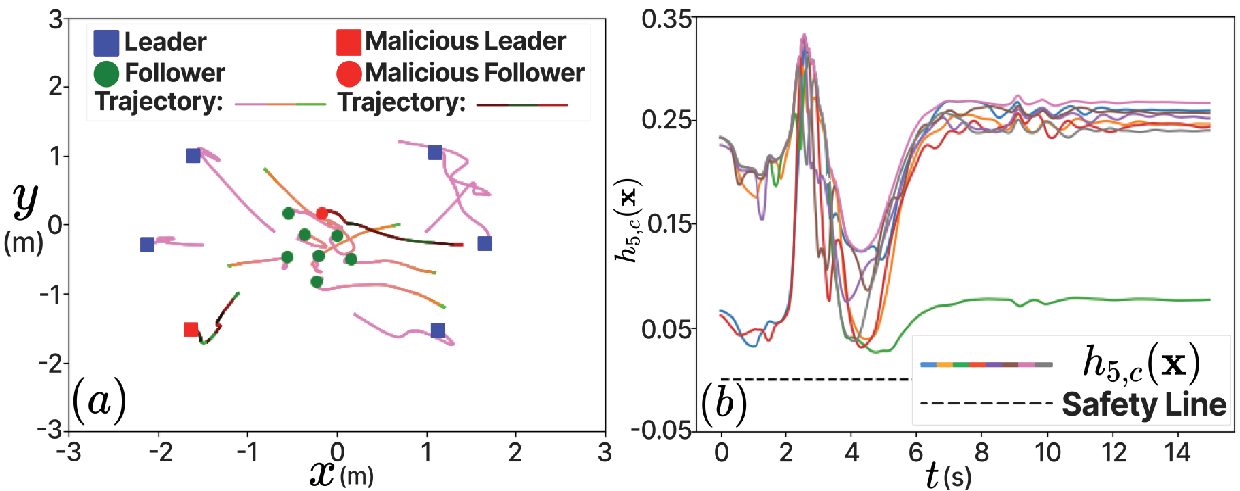}
    \caption{\small{(a) and (b) display the evolutions of robots' trajectories in their LED colors and $h_{5}(\mathbf x)$ \eqref{eq:all_h} from the first simulation, respectively.}}
    \label{fig:first}
\end{figure}
\begin{figure}
    \centering
\includegraphics[width=.48\textwidth]{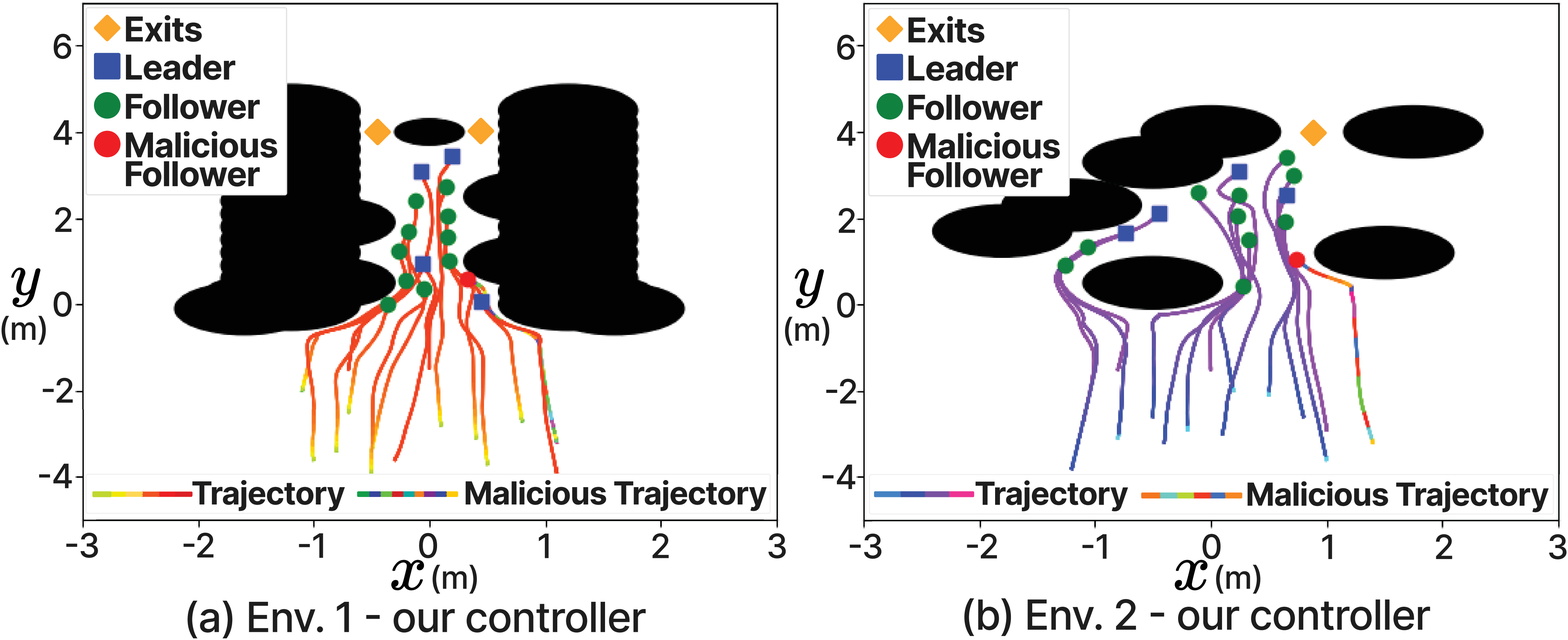}
    \caption{\small{(a) and (b) show the snapshots of the simulations with our controller and dynamics \eqref{eq:combined_dynamics} in Env. 1 and 2, respectively.}}
    \label{fig:env}
\end{figure}
\begin{figure}
    \centering
\includegraphics[width=.49\textwidth]{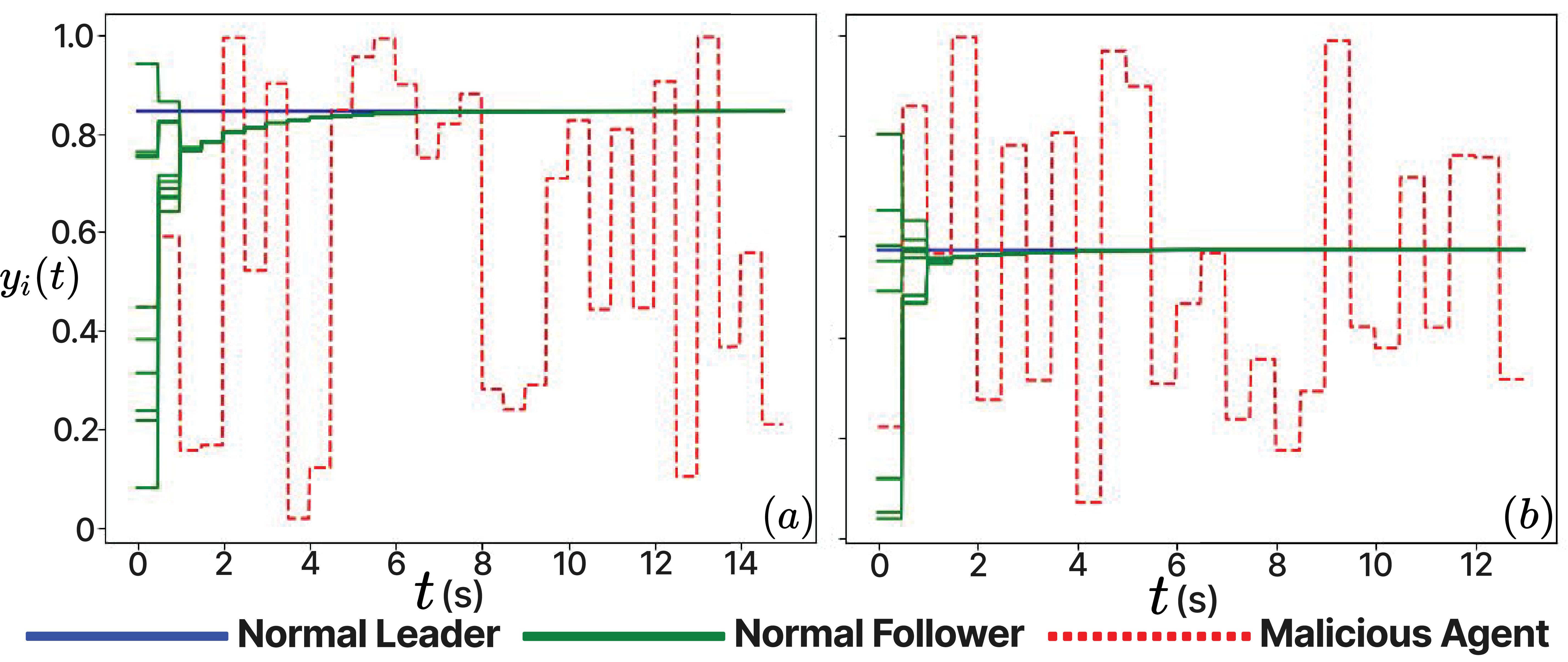}
    \caption{\small{(a) and (b) show the evolutions of robots' consensus values for the simulations visualized in Fig.~\ref{fig:env} (a) and (b), respectively.}}
    \label{fig:consensus}
    \vspace{-5mm}
\end{figure}

\definecolor{lightgray}{gray}{0.7}
\begin{table}
    \centering
    \begin{tabular}{|l|c|c|}
      \hline & Env. $1$ & Env. $2$  \\\hline
    Baseline $1$ &   $47.65$ (Failed $36$ times) & 45.57 (Failed $18$ times)\\ \hline
   Baseline $2$   & $13.46$ & $12.42$\\   \hline
\cellcolor{lightgray}Ours  &  \cellcolor{lightgray}$11.72$ &  \cellcolor{lightgray}$11.82$\\    \hline
\end{tabular}
\caption{\small{Time (in seconds) for controllers to reach exits in two environments from Fig.~\ref{fig:env}. For Baseline $1$, we record the average performance of successful trials out of $50$ attempts with randomly generated topologies. Trials with $t\geq 50$ were considered failures.}}
\label{tab:records}
\end{table}

\textbf{1) Spread Out:} We show that our CBF can maintain strong $r$-robustness, even when the desired control input directly conflicts with maintaining the network’s robustness. Let $p_{i,\text{goal}}$ be a goal location for agent $i$. We set $\mathbf u_{\text{des}}$ as
\begin{equation}
u_{i,\text{des}}
    \begin{cases}
        \frac {(p_{i,\text{goal}}-p_i)}{||p_{i,\text{goal}}-p_i||} - v_i, & i \in \mathcal L, \\
        0, &  i\in \mathcal F,
    \end{cases}
    \label{eq:nominal_input}
\end{equation}
where we set $p_{i,\text{goal}}$ such that leaders would spread out in different directions. In the initial setup, we had six leaders (one malicious) and eight followers (one malicious) forming a strongly $6$-robust network with a communication range $R=3$ m. The system aimed to maintain strong $5$-robustness to ensure consensus under a $2$-local malicious attack. Fig.~\ref{fig:first} (a) shows the robots' trajectories colored in their LED colors, while Fig.~\ref{fig:first} (b) shows the evolutions of $h_5(\mathbf x)$ \eqref{eq:all_h}. The network remained strongly $5$-robust throughout the simulation, and followers successfully matched their LEDs to normal leaders'. Notably, when the network approached losing $5$-robustness at around $t = 4$ (Fig.~\ref{fig:first} (b)), the leaders stopped dispersal while the followers gathered at the center.

\textbf{2) Complex Environments:} We evaluated our controller in two complex environments visualized in Fig.~\ref{fig:env}. In the initial setup, four normal leaders and eleven followers (one malicious) formed a strongly $4$-robust network with $R=3$ m. Each robot $i$ had $u_{i,\text{des}}=\frac {p_{\text{goal}}-p_i}{||p_{\text{goal}}-p_i||}-v_i$ where $p_{\text{goal}}=[0,100]^T$. The system aimed to maintain strong $3$-robustness to achieve consensus under a $1$-local malicious attack. Fig.~\ref{fig:env} and \ref{fig:consensus} depict the robots' trajectories in their LED colors and consensus values, respectively. In both environments, the system maintained strong $3$-robustness, and normal followers matched their LED colors to those of the normal leaders.

\textbf{3) Comparisons:} To demonstrate the flexibility of our CBF in network formation, we compared its performance with two other CBFs: Baseline $1$ that maintains fixed topologies as in \cite{egerstedt2018} and Baseline $2$ from \cite{cavorsi2022}. For each, we measured the time for all robots to reach the exits in two environments (Fig.~\ref{fig:env}) while maintaining strong $3$-robustness. Although Baseline $2$ focuses on $r$-robustness which does not extend to strong $r$-robustness \cite{strongly_usevitch18, rezaee21}, we included it to maintain $3$-robustness for a more comprehensive comparison. For fairness, all comparisons were conducted with 1) the same parameters, 2) single integrator dynamics, as in \cite{cavorsi2022, egerstedt2018}, and 3) a centralized CBF-QP controller setup. A more flexible controller should lead to shorter exit times. Table~\ref{tab:records} displays the recorded times. Trials with $t\geq50$ were considered failures. Baseline $1$ was evaluated based on the average performance of successful trials across $50$ attempts. Each trial used strongly $3$-robust Erdös–Rényi random graphs \cite{erdos1964,bollobas2001} with $n=15$ nodes and an edge probability of $0.3$. Table~\ref{tab:records} shows our CBF allowed the robots to reach the exits fastest in both environments, significantly outperforming Baseline 1. This illustrates that our approach offers greater flexibility in maintaining strong $r$-robustness.

\subsection{Hardware Experiment}
We illustrate practical application of our method using Crazyflie (CF) 2.0 platform. Our setup consisted of eight CFs: four normal leaders and four followers (one malicious). Initially arranged to form a strongly $4$-robust network with a communication range of $R=2.5$ m, the CFs were to maintain strong $3$-robustness throughout the experiments. We used the Crazyswarm 1.0 ROS Neotic package \cite{crazyswarm} along with two Crazyradio PAs for communication and $15$ Vicon motion capture cameras for localization. Control inputs were computed offboard by solving QP \eqref{eq:total_qp} on a computer and sent to each CF. To prevent trivial formations where drones vertically align or fly at high altitudes to avoid obstacles, we fixed their altitudes at $0.7$m. We tested our proposed CBF on hardware by replicating the simulation scenarios. In the first scenario, we set $p_{i,\text{goal}}$ in \eqref{eq:nominal_input} such that four leaders would spread out in different directions. The next two experiments involved environments similar to Fig.\ref{fig:env} with 1) a narrow space and 2) multiple obstacles of varying sizes. In all cases, the drones successfully maintained strong $3$-robustness, and normal followers successfully matched their LED colors to those of normal leaders. Due to space constraints, results and videos are available \href{https://github.com/5217232/Resilient-Leader-Follower-CBF-QP-Controller}{here}$^{\ref{note1}}$.

\section{Conclusion}
\label{sec:conclusion}
This paper presents a Control Barrier Function (CBF) that guarantees the maintenance of a multi-robot network's strong $r$-robustness above some threshold. We construct a valid CBF by composing multiple HOCBFs, each of which addresses an agent's sufficient connectivity for the network's robustness. Our approach directly considers robustness without predetermined topologies, thus offering greater flexibility in network formation. We have evaluated our method with simulations and hardware experiments. 

\bibliographystyle{IEEEtran}






\bibliographystyle{IEEEtran}
\bibliography{references_ll}
\end{document}